\documentclass{article}
\usepackage{microtype}
\usepackage{graphicx}
\usepackage{subfigure}
\usepackage{booktabs}
\usepackage{hyperref}

\usepackage[accepted]{icml2024}
\usepackage{natbib}

\usepackage{amsmath}
\usepackage{amssymb}
\usepackage{mathtools}
\usepackage{amsthm}
\usepackage{multirow}
\usepackage{subfigure}
\usepackage{threeparttable}
\usepackage{color}
\usepackage{colortbl} 
\usepackage{newfloat}
\usepackage{listings}

\usepackage[capitalize,noabbrev]{cleveref}

\theoremstyle{plain}
\newtheorem{theorem}{Theorem}[section]
\newtheorem{proposition}[theorem]{Proposition}

\theoremstyle{definition}

\theoremstyle{remark}

\usepackage[textsize=tiny]{todonotes}

\begin{document}

\twocolumn[
\icmltitle{ParsNets: A Parsimonious Orthogonal and Low-Rank Linear Networks for Zero-Shot Learning}




\begin{icmlauthorlist}
\icmlauthor{Jingcai Guo}{polyu,polyu-sz}
\icmlauthor{Qihua Zhou}{polyu}
\icmlauthor{Ruibing Li}{polyu}
\icmlauthor{Xiaocheng Lu}{hkust}
\icmlauthor{Ziming Liu}{polyu}
\icmlauthor{Junyang Chen}{polyu}
\icmlauthor{Xin Xie}{tju}
\icmlauthor{Jie Zhang}{polyu}
\end{icmlauthorlist}

\icmlaffiliation{polyu}{Department of Computing, The Hong Kong Polytechnic University}
\icmlaffiliation{polyu-sz}{The Hong Kong Polytechnic University Shenzhen Research Institute}
\icmlaffiliation{hkust}{Department of Computer Science and Engineering, The Hong Kong University of Science and Technology}
\icmlaffiliation{tju}{School of Computer Science and Technology, Tianjin University}

\icmlcorrespondingauthor{\textbf{Jingcai Guo}}{\textbf{jc-jingcai.guo@polyu.edu.hk}}
%\icmlcorrespondingauthor{Firstname2 Lastname2}{first2.last2@www.uk}

%\icmlkeywords{Machine Learning, ICML}

\vskip 0.3in
]



\printAffiliationsAndNotice{}  % leave blank if no need to mention equal contribution

\begin{abstract}
This paper provides a novel parsimonious yet efficient design for zero-shot learning (ZSL), dubbed \textit{ParsNets}, where we are interested in learning a composition of \textit{on-device friendly} linear networks, each with orthogonality and low-rankness properties, to achieve equivalent or even better performance against existing deep models. Concretely, we first refactor the core module of ZSL, i.e., visual-semantics mapping function, into several base linear networks that correspond to diverse components of the semantic space, where the complex nonlinearity can be collapsed into simple local linearities. Then, to facilitate the generalization of local linearities, we construct a maximal margin geometry on the learned features by enforcing low-rank constraints on intra-class samples and high-rank constraints on inter-class samples, resulting in orthogonal subspaces for different classes and each subspace lies on a compact manifold.
%
%Moreover, 
To enhance the model's adaptability and counterbalance over/under-fittings in ZSL, a set of sample-wise indicators is employed to select a sparse subset from these base linear networks to form a composite semantic predictor for each sample. 
Notably, maximal margin geometry can guarantee the diversity of features, and meanwhile, local linearities guarantee efficiency. Thus, our \textit{ParsNets} can generalize better to unseen classes and can be deployed flexibly on resource-constrained devices. 
Theoretical explanations and extensive experiments are conducted to verify the effectiveness of the proposed method.
\end{abstract}

\begin{figure}[h]
%\vspace{-3.5mm}
\centering
\subfigure[2D Subspace Map]{\label{m3-original}
\includegraphics[width=1.5in]{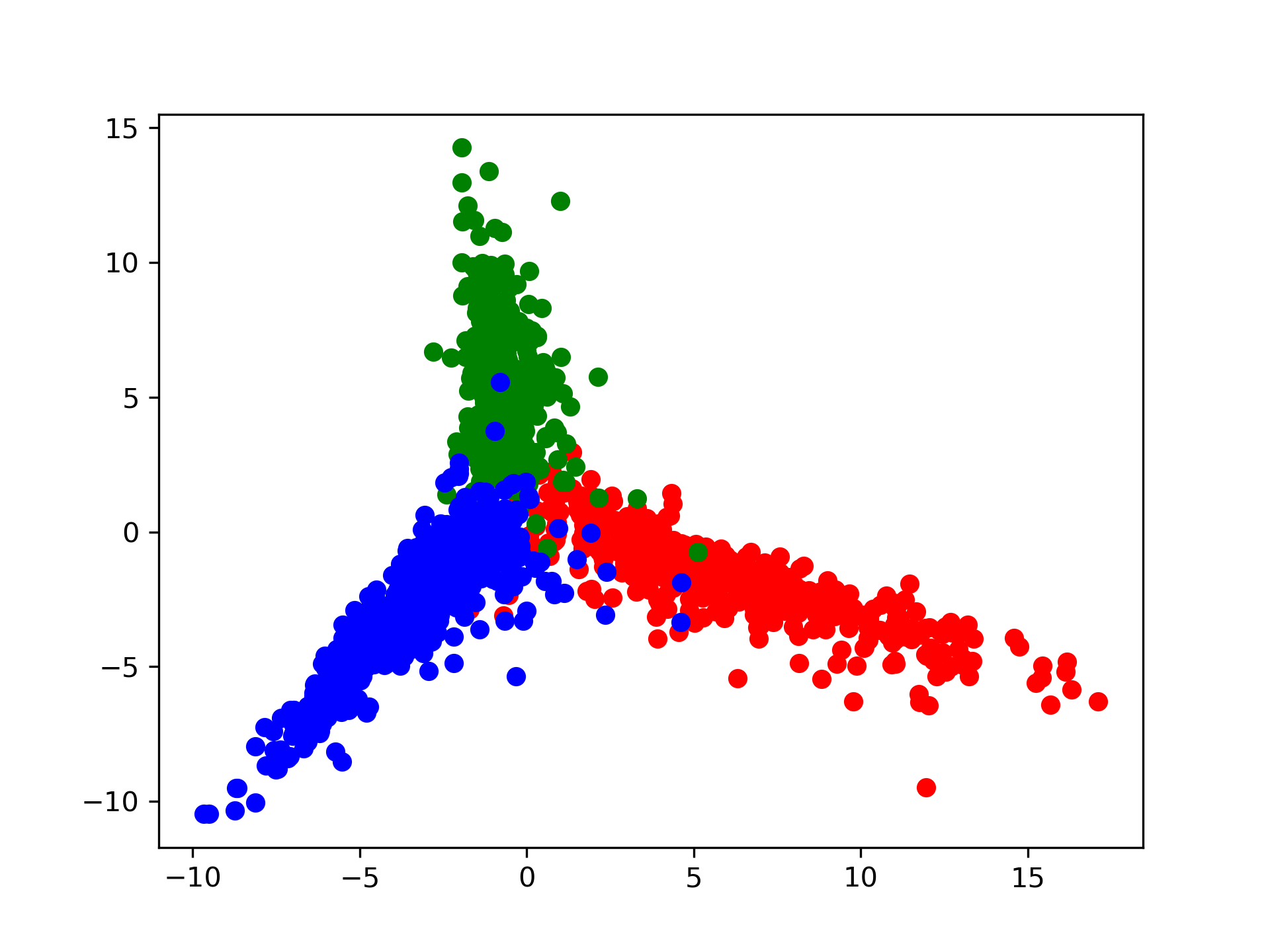}}
\subfigure[3D Subspace Map]{\label{m3-mapped}
\includegraphics[width=1.5in]{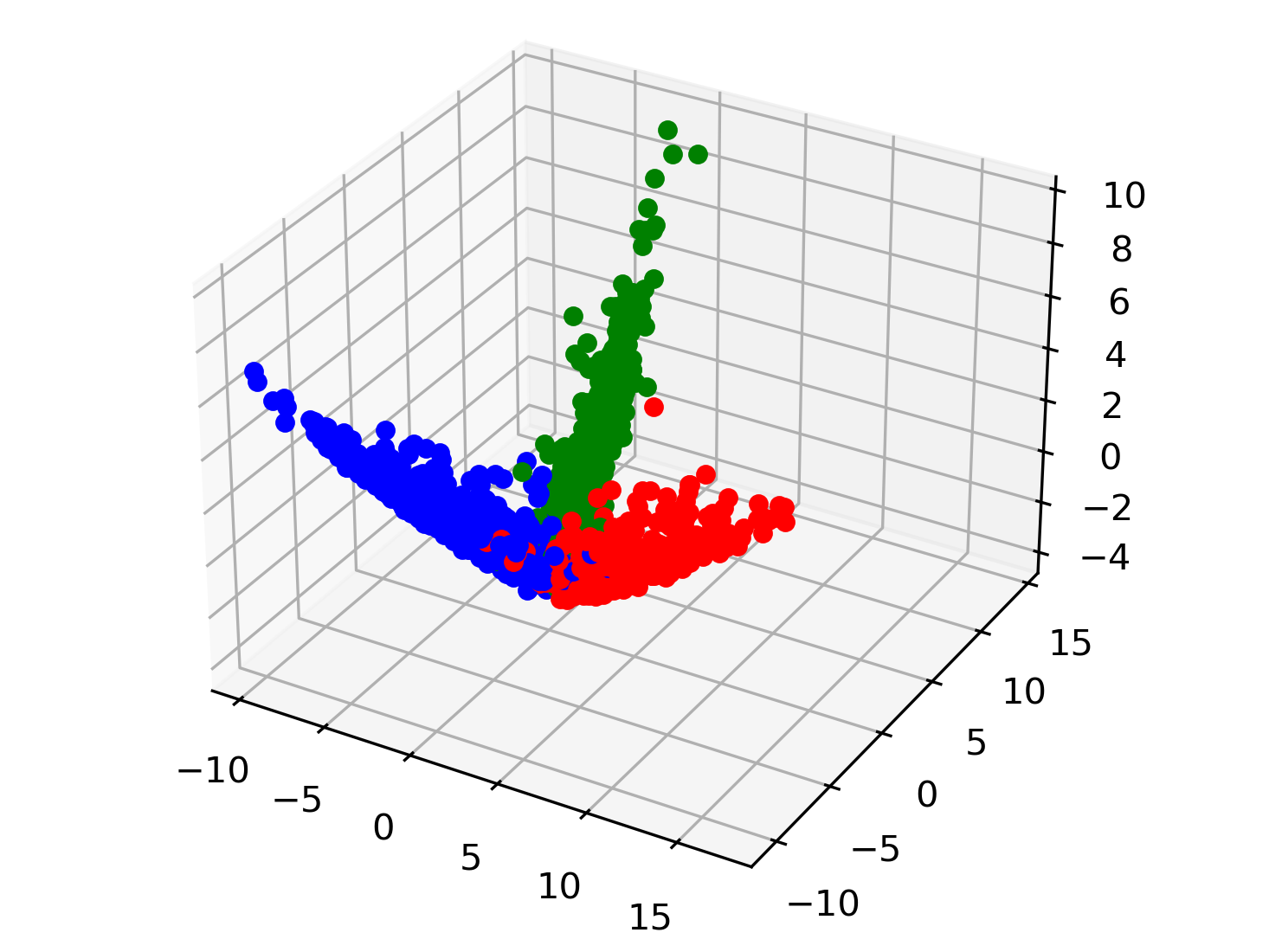}}
%\vspace{-2mm}
\caption{Subspace visualization of AWA2~\cite{xian2018zero} on three random classes. The orthogonality indicates the maximum separability.} 
%\vspace{-3.5mm}
\end{figure}

\section{Introduction}
%1. descript ZSL
%Zero-shot learning (ZSL) has received increasing attention recently for its imitation ability of human-like knowledge transfer to recognize unseen classes without having to observe any real samples before, i.e., zero-shot training samples of unseen classes~\cite{kodirov2017semantic,zhu2019semantic,chen2021semantics,khan2023learning}. 
Zero-shot learning (ZSL) has received increasing attention for its imitation ability of human-like knowledge transfer to recognize unseen classes without having to observe any real sample before~\cite{chen2021semantics,khan2023learning}.
Such recognition is typically achieved by training labeled seen class samples combined with a set of shared semantic descriptors spanning both seen and unseen classes, and generalizing the trained model to recognize samples from unseen classes. 
%
%In practice, 
The shared semantic descriptors are usually implemented by simple semantic attributes~\cite{lampert2013attribute} or word vectors~\cite{pennington2014glove} that contain class-wise high-level information for each class. Therefore, a natural and widely adopted ZSL solution is to map a sample from its original feature space, e.g., visual space w.r.t. images, to the shared semantic space to construct a visual-semantics mapping function, wherein, the mapped semantic representation is calculated with those descriptors to search for one matched class that has the highest compatibility with the sample.
%, whose corresponding class is then assigned to the sample. 
%attributes~\cite{lampert2013attribute} or word vectors from GloVe~\cite{pennington2014glove}.
%Notably, classic ZSL only infers to unseen classes during testing, while generalized ZSL is expected to accurately classify samples from both seen and unseen classes. 
Notably, based on the search scope, the classic ZSL and generalized ZSL (GZSL) are further defined, where the former infers only unseen classes, while the latter can also classify novel samples of seen classes. 

% \begin{figure}[t]
%   \centering
% \centerline{\includegraphics[width=0.3\textwidth]{cp1_new.png}}
%   \caption{Model Complexity v.s. ZSL Accuracy}
%   \label{fig1}
% %\vspace{-5mm}
% \end{figure}

%2. from domain bias -> over/under-fittings -> better representation
%3. from large pre-train -> costly -> lightweight 
In ZSL/GZSL, since the trained model has no observation of any sample from unseen classes, the mapping function is inherently biased towards seen classes~\cite{fu2015transductive,guo2023graph}, i.e., the mapped features are usually overfitted to clusters near seen classes, hence can hardly infer to unseen classes with satisfactory performance. 
%https://ieeexplore.ieee.org/stamp/stamp.jsp?tp=&arnumber=9832795
To relieve the domain-biased overfitting, especially for GZSL, existing methods usually resorted to learning more representative features having less gap between seen/unseen classes. Some widely used approaches include 1) visual-semantics alignment, for smooth knowledge transfer~\cite{schonfeld2019generalized,guo2020novel}; 2) generative methods (i.e., synthesize mimic samples of unseen classes), for training a holistic model~\cite{huang2019generative,zhao2022boosting}; and 3) fine-grained methods, for extracting more generalizable features~\cite{huynh2020fine,guo2023graph}. 
%

% For instance, some methods tried to align the visual and semantic features within the visual-semantics mapping function to enable smooth knowledge transferring~\cite{zhang2015zero,schonfeld2019generalized,guo2020novel}. 
% %between seen/unseen classes~\cite{zhang2015zero,schonfeld2019generalized,guo2020novel}. 
% %
% Differently, some other methods proposed to synthesize mimical samples conditioned on unseen class semantic descriptors and jointly train a holistic model combined with seen class samples~\cite{huang2019generative,zhao2022boosting,feng2022non}. 
% %
% In contrast, some recent methods focused on fine-grained elements or key points in samples to extract more generalizable features between seen/unseen classes~\cite{xie2019attentive,huynh2020fine,guo2023graph}. 
%

Despite certain relief from the domain bias, we observe that existing methods, in turn, may inevitably incur the underfitting phenomenon on their trained models. 
For example, most GZSL models usually tend to obtain higher Harmonic Mean metrics~\cite{xian2017zero} of the test accuracy. In other words, the recognition tasks for seen and unseen classes can suppress each other, yielding two mediocre results across seen/unseen classes. 
Hence, the training of ZSL/GZSL can easily fluctuate between over/under-fittings and degrade the model's generalization ability. 
%easily suffer from the fluctuation between \textit{overfitting} and \textit{underfitting} and degrade the performance dramatically.
%

Moreover, it can be noted that existing methods mostly rely on a series of complex deep models to extract and fuse comprehensive features for superior recognition~\cite{xie2019attentive,huynh2020fine,wang2018zero,guo2023graph}. 
As a result, such training and deployment can be costly in terms of both computing and memory overhead. 
In this regard, we make an assumption that ZSL/GZSL can be more favorable to a scenario associated with resource-constrained devices due to the low or even zero data requirements. Such a scenario can also well align with ubiquitous devices and data in real-world applications. 
However, as far as we know, nearly no research has investigated lightweight ZSL/GZSL models deployed on resource-constrained devices. 

In this paper, we suggest that the above generalization and lightweight requirements can be jointly achieved by our properly designed parsimonious-yet-efficient network refactoring framework, namely, \textit{ParsNets}. 
Specifically, we utilize a set of base linear networks to estimate the nonlinear visual-semantics mapping function that usually involves complex deep models, wherein, each base linear network can correspond to different components of the semantic space shared by both seen and unseen classes. 
To encourage the learned features to be most discriminative from each other and generalizable to novel concepts, we enforce a low-rank structure to the features of data samples from the same class and a high-rank structure to the features of data samples from all different classes. Hence, intra-class samples can reside in the same linear subspace, and meanwhile, inter-class subspaces can be orthogonal from each other. Such constructed maximal margin geometry is expected to facilitate a smooth knowledge transfer between seen and unseen classes since no entanglement exists. 
Moreover, to further encourage the model's adaptability and counterbalance over/under-fittings, we employ a set of sample-wise indicators to select a sparse subset of these base linear networks to form a composite predictor for each sample, thus the global nonlinearity can be collapsed into sparse local linearities to further reduce the computing complexity. 
%
%Our framework is on-device friendly and can better generalize to unseen classes.

In summary, our contributions are three-fold:
%In summary, our contributions are four-fold:
\begin{itemize}
\item We propose \textit{ParsNets}, which is the first work that provides a parsimonious and on-device-friendly framework for ZSL/GZSL by refactoring the nonlinear large network into a composition of simple local linear networks. %which can be deployed flexibly on resource-constrained devices.
\item We enforce maximal margin geometry on the learned features to maximize the model's discrimination and generalization ability, thus enabling a smooth knowledge transfer between seen and unseen classes.
\item We provide detailed theoretical explanations on the rationality and implementation guarantee of \textit{ParsNets}, which indicates its feasibility.
%\item Experimental results on ZSL/GZSL tasks demonstrate the effectiveness of the proposed method.
\end{itemize}

% In this paper, we propose to refactor the mapping function into several orthogonal \textit{base linear networks} that correspond to diverse components of the semantic space along with low-rank constraints on intra-class features and high-rank constraints on inter-class features. To enhance the model's adaptability and counterbalance over/under-fittings, a set of sample-wise indicators are employed to select a \textit{sparse subset} from these base linear networks to form a \textit{composite predictor} for each sample.
% %
% Notably, orthogonality and low-rankness can guarantee the diversity of features, and meanwhile, linearity guarantees efficiency. Thus, our model generalizes better to unseen classes and can be deployed flexibly on resource-constrained devices. 
% %
% Theoretical explanations and extensive experiments are conducted to verify the effectiveness of our method.
\section{Related Work}

\subsection{Visual-Semantics Mapping in ZSL/GZSL}
Existing ZSL/GZSL methods adopt three approaches to construct the visual-semantics mapping, including forward, reverse, and intermediate functions. 
Among them, forward mapping is the mainstream that maps samples from their visual features to the semantic space and computes their compatibilities with class-level semantic descriptions~\cite{akata2015evaluation,schonfeld2019generalized}. 
Conversely, reverse mapping suggests that projecting semantic features into the visual space may decrease the feature variances~\cite{zhang2017learning}. 
Diverging from direct mappings, intermediate functions explore metric networks to compute compatibilities of paired input visual and semantic features in an intermediate space~\cite{sung2018learning}. 
However, we can note that the above methods are mostly built upon complex deep models which are all innately computationally and memory-intensive frameworks.
%There still a 
%Existing state-of-the-art ZSL methods are mostly built upon large pre-trained networks (e.g., various vision-language models) combined with comprehensive feature extraction units (e.g., attention or graph embedding) to construct a visual-semantics mapping, wherein seen and unseen classes reside in a shared space. 
%
%Existing ZSL methods are mostly built upon large networks with comprehensive feature extraction units to train a visual-semantics mapping function. 
%wherein seen/unseen classes share the same set of auxiliary knowledge (e.g., semantic descriptors). 
%
%However, deep models are innately computationally and memory-intensive, and worse still, 
%require a large overhead in terms of both computation and memory.
%are computationally intensive and not on-device friendly, 
%
\noindent[\textbf{This paper}]: To align with most methods, we follow the forward mapping approach to implement the visual-semantics mapping function, and moreover, construct the followed network refactoring to achieve a parsimonious-yet-efficient ZSL framework.

\subsection{Domain-Bias Problem}
%Given the nature that the training (seen) and testing (unseen) classes are disjoint, domain bias can be an inherent issue in ZSL when transferring learned knowledge between seen/unseen classes~\cite{fu2015transductive}. 
%
In recent years, extensive efforts have been made to relieve the domain bias caused by the disjoint train (seen) and test (unseen) classes. 
For example, some methods tried to align the visual and semantic features within the visual-semantics mapping function to enable a smooth knowledge transfer~\cite{zhang2015zero,schonfeld2019generalized,guo2020novel}. 
Differently, some other methods proposed to synthesize mimical samples conditioned on unseen class semantic descriptors and jointly train a holistic model combined with seen class samples~\cite{huang2019generative,zhao2022boosting,feng2022non}. 
In contrast, some recent methods focused on fine-grained elements or key points within samples to extract more generalizable features between seen/unseen classes~\cite{xie2019attentive,huynh2020fine,guo2023graph}. 
%
%Despite some relief from the problem of domain-biased overfitting, we observe that underfitting can also arise and fluctuate the training towards mediocre convergence for both seen and unseen classes.
Despite the relief from the domain-biased overfitting, we observe that underfitting can arise and fluctuate the training towards mediocre convergence for both seen/unseen classes.
\noindent[\textbf{This paper}]: We construct the maximal margin geometry along with the network refactoring to encourage the discrimination of learned features, and utilize sample-wise indicators to enable a sparse composition of base linear networks. 
Such regularization can facilitate the generalization towards unseen classes, and explore a better balance between over/under-fittings.

\section{Methodology}

%\subsection{Problem Definition}
% \subsection{Preliminaries}
% \begin{figure*}[t]
%   \centering
% \centerline{\includegraphics[width=0.77\textwidth]{framework-zsl.png}}
%   \caption{A.}
%   \label{fig2}
% %\vspace{-5mm}
% \end{figure*}

%\subsection{Problem Definition}
\subsection{Preliminaries}
Given a dataset from seen domain $\mathcal{D^{S}}=\left \{ x_{i}, a_{y_{i}}, y_{i} \right \}_{i=1}^{N}$ that contains $N$ labeled samples $x_{i} \in \mathcal{X^{S}}$ with seen class labels $y_{i} \in \mathcal{Y^{S}}$. 
The task of GZSL is to construct a model trained on $\mathcal{D^{S}}$ while can also generalize well to unseen domain $\mathcal{D^{U}}=\left \{ \left (x^{u}, a_{y^{u}}, y^{u}\right ) \mid x^{u} \in \mathcal{X^{U}}, y^{u} \in \mathcal{Y^{U}}\right \}$, where $\mathcal{X^{S}} \cap \mathcal{X^{U}} = \phi$ and $\mathcal{Y^{S}} \cap \mathcal{Y^{U}} = \phi$. 
To achieve this goal, a set of shared per-class semantic descriptors $\mathcal{A} = \mathcal{A^{S}} \cup \mathcal{A^{U}}$ is further specified for seen classes ($a_{y} \in \mathcal{A^{S}}$) and unseen classes ($a_{y^{u}} \in \mathcal{A^{U}}$), respectively, which are widely implemented by attributes or word vectors.
Thus, the model can be formalized into training a parameterized mapping function $f:\mathcal{X} \to \mathcal{A}$ that maps a sample $x$ from its original feature space $\mathcal{X}$, e.g., visual space w.r.t. images, to the shared semantic space $\mathcal{A}$ as:
\begin{equation}
\mathop{\arg \min}_{\mathbf{\Theta}} \ \frac{1}{N} \sum_{i=1}^{N} \mathcal{L} \left ( f \left ( x_{i}; \mathbf{\Theta} \right ), a_{y_{i}} \right ) + \varphi \left ( \mathbf{\Theta} \right ),
\label{vs-mapping}
\end{equation}
where $\mathcal{L}\left ( \cdot  \right )$ minimizes the variance between mapped semantic features and ground-truth descriptors and $\varphi \left ( \cdot \right )$ regularizes the network weight $\mathbf{\Theta}$, if needed.
During inference, given a test sample $x^{t}$, the recognition can be described as:
\begin{equation}
\mathop{\arg\max}_{j} \ \Lambda \left ( f\left (x_{t}; \mathbf{\Theta} \right ), a_{j} \mid \mathcal{A} \right ),
\end{equation}
where $\Lambda \left (\cdot \right )$ is a similarity metric that searches the most closely related descriptor from $\mathcal{A}$, whose class is then assigned to the sample. 
Notably, if the search scope is restricted to $\mathcal{A^{U}}$, the recognition becomes the classic ZSL.
% \begin{align}
% &\mathop{\arg\max}_{j} \ \Lambda \left ( f\left ( \psi \left ( x_{t} \right ); \mathbf{W} \right ), {{P}'}^{(j)} \right ), \\
% &\mathop{\arg\max}_{j} \ \Lambda \left ( f\left ( \psi \left ( x_{t} \right ); \mathbf{W} \right ), \left \{ {{P}' \cup P} \right \}^{(j)} \right ), 
%\end{align}
% where Eq. (2) works for the conventional ZSL task where the similarity search is limited to unseen classes, and Eq. (3) is for the generalized ZSL task where the search can generalize to new samples from seen classes as well.

\subsection{Adaptive Composition of Linear Networks}
To break the dependence on complex deep models, one possible solution is to use local linearities to estimate the global nonlinearity~\cite{liu2018unified,li2020field}. 
In this paper, we apply a similar strategy to construct the visual-semantics mapping function. Concretely, we assume that a complex nonlinearly separable space can have sufficient linearly separable subspaces, and as such, we refactor the deep function $f \left ( \cdot; \mathbf{\Theta}  \right )$ into several base linear networks as:
%derived from the assumption that a complex nonlinearly separable space can have sufficient linearly separable subspaces~\cite{liu2018unified,li2020field}. 
%
%In this paper, we apply a similar strategy to construct the visual-semantics mapping function, that is, refactoring the deep function $f \left ( \cdot; \mathbf{\Theta}  \right )$ into several base linear networks, i.e., 
%Suppose the deep mapping function $f \left ( \cdot; \mathbf{W}  \right )$ is refactored into $K$ base linear networks as:
\begin{equation}
f(x) = \sum_{i=1}^{K} \left ( \Theta_{i}^\mathsf{T}x + b_{i}\right ) \cdot  \xi_{i} (x) + b, 
\label{fx}
\end{equation}
where $\Theta_{i}^\mathsf{T}x + b_{i}$ ($i=1, 2, \dots, K$) are a set of base linear networks with trainable weight $\Theta_{i}^\mathsf{T}$ that jointly estimate $f(x)$, and $b_{i}$ denotes the bias constant vector. Such an estimation directly reduces the overhead in terms of both computing and memory with only limited accuracy loss~\cite{oiwa2014partition,liu2018unified,li2020field}. 

Differently, in our method, we further utilize a set of associated indicators $\xi_{i} (\cdot)$, i.e., 
\begin{equation}
\mathbf{\Xi}\left ( x \right ) = \left [ \xi_{1} (x), \dots,  \xi_{K} (x) \right ], 
\label{indicators}
\end{equation}
to adaptively select a subset of $\left \{ \Theta_{i}^\mathsf{T} \right \}$ to form a composite sample-wise semantic predictor $f(x) \mid \xi_{i} (x)$ for each $x$. 
%binary
%To enable the adaptability of the model on diverse samples in the semantic space, a series of associated indicators, i.e., $\xi_{i} (\cdot)$, are further implemented to provide sparse binary signals that determine a subset of $\left \{ \Theta_{i}^\mathsf{T} \right \}$ to form one composite predictor $f(x)$ for each sample $x$. 
%
Such indicators are expected to be binary and sparse signals to fit with diverse samples in both visual and semantic spaces. Notably, the binarity can provide a gated functionality to active specific $\Theta_{i}^\mathsf{T}$, and the sparsity pushes the composite predictor to explore a better balance between over/under-fittings due to sample-wise selection. 
Without loss of generality, given the $\xi_{i} (\cdot)$ (introduced in sect.~\textit{Sample-Wise Indicators}), we define one variable $\mathbf{\Phi}(x)$ as:
\begin{equation}
\begin{split}
\mathbf{\Phi}(x) &= \left [ \xi_{1} (x), x^\mathsf{T}\xi_{1} (x), \cdots , \xi_{K} (x), x^\mathsf{T}\xi_{K} (x) \right ], \\
&=\left [\mathbf{\Xi}^\mathsf{T}(x) \otimes \left [ \mathbf{1} \ x^\mathsf{T} \right ]  \right ]^\mathsf{T},
\end{split}
\label{Phi}
\end{equation}
where $\mathbf{\Xi}^\mathsf{T}(x) = \left [ \xi_{1} (x), \cdots ,\xi_{K} (x) \right ]$ and $\otimes$ denotes the Kronecker production, 
and denote the variable $\mathbf{\Theta}$ as:
\begin{equation}
\mathbf{\Theta} = \left [ b_{1}, \Theta_{1}^\mathsf{T}, \cdots , b_{K}, \Theta_{K}^\mathsf{T} \right ]^\mathsf{T}.
\end{equation}
Thus, the composite semantic predictor in Eq.~\ref{fx} can be rewritten as a regression form as:
\begin{equation}
f(x) = \mathbf{\Theta}^\mathsf{T}\mathbf{\Phi}(x) + b.
\label{simplified}
\end{equation}
Note that Eq.~\ref{simplified} can be significantly reduced to a lightweight model with the sparse binary signals from $\mathbf{\Xi}^\mathsf{T}(x)$, wherein, only a few numbers of $\left \{ \Theta_{i}^\mathsf{T} \right \}$ are activated. 
In other words, the nonlinear model $f(x)$ can collapse into a linear model at each point $x$ and its nearby local field. 

Mathematically, Eq.~\ref{simplified} can be solved efficiently by constructing a quadratic programming problem with the structural risk minimization as:
\begin{small}
\begin{equation}
\begin{split}
&\mathop{\arg \min}_{\mathbf{\Theta}, b, \gamma_{l}, \gamma_{l}^{*}} \ \frac{1}{2} \mathbf{\Theta}^\mathsf{T} \mathbf{\Theta} + C\sum_{l=1}^{N}\left ( \gamma_{l} + \gamma_{l}^{*} \right ), \\
&\mathrm{s.t.}
\left\{\begin{matrix}
\mathbf{\Theta}^\mathsf{T}\mathbf{\Phi}(x_{l}) + b - a_{y_{l}} \le \epsilon + \gamma_{l}^{*} \\
a_{y_{l}} - \mathbf{\Theta}^\mathsf{T}\mathbf{\Phi}(x_{l}) - b \le \epsilon + \gamma_{l} \\
\gamma_{l}, \gamma_{l}^{*} \ge 0, i = 1,2,\cdots, N
\end{matrix}\right.,
\end{split}
\label{problem}
\end{equation}
\end{small}
where $a_{y_{l}}$ is the expected semantic descriptor, $\gamma_{l}$ and $\gamma_{l}^{*}$ are slack variables, and $C$ is a non-negative weight that penalizes the prediction error $\epsilon$. 
Since the input dimension is usually large, we can consider its dual form. Specifically, we introduce Lagrange multipliers $\alpha_{l} \ge 0$, $\mu_{l} \ge 0$, $\alpha_{l}^{*} \ge 0$, and $\mu_{l}^{*} \ge 0$ to obtain the Lagrange function as:
\begin{small}
\begin{equation}
\begin{split}
&L\left (\mathbf{\Theta}, \gamma_{t}, \gamma_{t}^{*}, \alpha, \alpha^{*}, \mu, \mu^{*}\right ) = \frac{1}{2} \mathbf{\Theta}^\mathsf{T} \mathbf{\Theta} + C\sum_{l=1}^{N}\left ( \gamma_{l} + \gamma_{l}^{*} \right ) \\
&+ \sum_{l=1}^{N} \alpha_{l}\left (\underset{f(x_{l})}{\underbrace{\mathbf{\Theta}^\mathsf{T}\mathbf{\Phi}(x_{l}) + b}} - a_{y_{l}} - \epsilon - \gamma_{l}\right ) \\
&+ \sum_{l=1}^{N} \alpha_{l}^{*} \left (\underset{-f(x_{l})}{\underbrace{-\mathbf{\Theta}^\mathsf{T}\mathbf{\Phi}(x_{l}) - b}} + a_{y_{l}} - \epsilon - \gamma_{l}^{*}\right ) \\
&- \sum_{l=1}^{N}\left (\mu_{l}\gamma_{l} + \mu_{l}^{*}\gamma_{l}^{*}\right ).
\end{split}
\end{equation}
\end{small}
% \begin{equation}
% \begin{split}
% &L\left (\mathbf{\Theta}, \gamma_{t}, \gamma_{t}^{*}, \alpha, \alpha^{*}, \mu, \mu^{*}\right ) = \frac{1}{2} \mathbf{\Theta}^\mathsf{T} \mathbf{\Theta} + C\sum_{l=1}^{N}\left ( \gamma_{l} + \gamma_{l}^{*} \right ) \\
% &+ \sum_{l=1}^{N} \alpha_{l}\left (f(x) - a_{y_{l}} - \epsilon - \gamma_{l}^{*}\right ) \\
% &+ \sum_{l=1}^{N} \alpha_{l}^{*} \left (-f(x) + a_{y_{l}} - \epsilon - \gamma_{l}^{*}\right ) \\
% &- \sum_{l=1}^{N}\left (\mu_{l}\gamma_{l} + \mu_{l}^{*}\gamma_{l}^{*}\right ).
% \end{split}
% \end{equation}
It is easy to note that such a function can be solved by obtaining the saddle point $\frac{\partial L}{\partial \mathbf{\Theta}}=0$, $\frac{\partial L}{\partial \gamma_{l}}=0$, and $\frac{\partial L}{\partial \gamma_{l}^{*}}=0$, and we can rewrite Eq.~\ref{problem} as:
\begin{small}
\begin{equation}
\begin{split}
\mathop{\arg \max}_{\alpha, \alpha^{*}} \ &\sum_{l=1}^{N}\left (\alpha_{l}-\alpha_{l}^{*}\right ) f(x) - \varepsilon \sum_{l=1}^{N}\left (\alpha_{l}+\alpha_{l}^{*}\right ) \\
&-\frac{1}{2}\sum_{l=1}^{N}\sum_{j=1}^{N}\left (\alpha_{l}-\alpha_{l}^{*}\right ) \left (\alpha_{j}-\alpha_{j}^{*}\right ) T \left ( x_{l}, x_{j} \right ), \\
&\mathrm{s.t.}
\sum_{l=1}^{N}\left (\alpha_{l} - \alpha_{l}^{*}\right ) = 0, 
0 \le \alpha_{l}, \alpha_{l}^{*} \le C,
% \left\{\begin{matrix}
% \sum_{l=1}^{N}\left (\alpha_{l} - \alpha_{l}^{*}\right ) = 0  \\
% 0 \le \alpha_{l}, \alpha_{l}^{*} \le C
% \end{matrix}\right.,
\end{split}
\end{equation}
where $T \left ( x_{l}, x_{j} \right )$ is a constructed transformation function defined via Eq.~\ref{Phi} as:
\begin{equation}
\begin{split}
T \left ( x_{l}, x_{j} \right ) &= \mathbf{\Phi}(x_{l})^\mathsf{T} \mathbf{\Phi}(x_{j}) \\
&=\left ( 1 + x_{l}x_{j} \right ) \sum_{i=1}^{M} \xi_{i}(x_{l}) \xi_{i}(x_{j}).
\end{split}
\end{equation}
\end{small}
Solving this problem and obtaining $\alpha_{l}$ and $\alpha_{l}^{*}$, the model in Eq.~\ref{fx} can then be described as:
\begin{equation}
f\left ( x \right ) = \sum_{l=1}^{N}\left ( \alpha_{l} - \alpha_{l}^{*} \right ) T \left ( x, x_{l} \right ) + b.
\end{equation}

\subsection{Sample-Wise Indicators}
\label{sect-indicators}
As to the associated indicators $\mathbf{\Xi}\left ( x \right )$ in Eq.~\ref{indicators}, we expect such signals can have the following properties: 1) discrimination, which can provide sufficient compositions of linear networks for diverse samples; and 2) sparsity and binarity, which activate only a small subset of base linear networks to form the composite predictor. 
Based on this guidance, our method employs a simple unsupervised linear encoder (denoted as $\mathbf{W}$)-decoder (denoted as ${\mathbf{W}}'$) network to preliminarily capture the intrinsic data structure:
\begin{equation}
\mathop{\arg \min}_{\mathbf{W}, {\mathbf{W}}'} \ \left \| \mathbf{X} - {\mathbf{W}}'\mathbf{W}\mathbf{X} \right \|_{2},
\label{encoder}
\end{equation}
wherein, the latent embedding $\mathbf{E}=\mathbf{W}\mathbf{X}$ is usually more representative compressed variables that can be potentially used to construct the indicators $\mathbf{\Xi}\left ( x \right )$. In practice, Eq.~\ref{encoder} can be reformulated as $\left \| \mathbf{X} - \mathbf{W}^{\mathsf{T}} \mathbf{W}\mathbf{X} \right \|_{2}$ by using the tied weights~\cite{ranzato2007sparse}, i.e., ${\mathbf{W}}' = \mathbf{W}^{\mathsf{T}}$, where only $\mathbf{W}$ remains for estimation, hence reducing the complexity. 

Now we elaborate on the design of $\mathbf{\Xi}\left ( x \right )$. Given the embedding $\mathbf{E}_{x} \in \mathbb{R}^{h}$ of a sample $x$, we split it into a set of $K$ components: 
\begin{small}
\begin{equation}
\left \{ \mathbf{E}_{x}^{(i)} = \mathbf{E}_{x}\left [ (i-1)\frac{h}{K} + 1, i\frac{h}{K} \right ] \in \mathbb{R}^{\frac{h}{K}} \right \}_{i=1}^{K},
\end{equation}
\end{small}
%$\left \{ \mathbf{E}_{x}^{(i)} = \mathbf{E}_{x}\left [ (i-1)\frac{h}{K} + 1, i\frac{h}{K} \right ] \in \mathbb{R}^{\frac{h}{K}} \right \}_{i=1}^{K}$ 
that correspond to $K$ base linear networks. 
To determine the selection, we calculate the variance of each $\mathbf{E}_{x}^{(i)}$ to the mean value of $\mathbf{E}_{x}$, i.e., denoted as $\mathrm{Var}^{(i)} \left ( \mathbf{E}_{x}^{(i)} \mid {\mu}_{E_{x}}\right )$.
Then, we can rewrite the indicators of Eq.~\ref{indicators} to: 
\begin{equation}
\mathbf{\Xi}\left ( x \right ) = \left [ \mathrm{Var}^{(1)}, \dots , \mathrm{Var}^{(K)} \right ],
\label{final-indicators}
\end{equation}
%$\mathbf{\Xi}\left ( x \right ) = \left [ \mathrm{Var}^{(1)}, \dots , \mathrm{Var}^{(K)} \right ]$, 
where each $\xi_{i} (x) = \mathrm{Var}^{(i)}$. 
The rationality lies in that, if the variance of $\mathbf{E}_{x}^{(i)}$ is large, then the latent variables are more significant compared with others. Finally, to achieve a sparse subset of base linear networks, we rank all variances and select top-$k$ indicators $\xi_{i} (x)$, i.e., $k \ll K$, and set them to 1, which can activate the corresponding base linear networks in Eq.~\ref{fx} with $\left ( \Theta_{i}^\mathsf{T}x + b_{i}\right ) \cdot 1$. Meanwhile, the remaining indicators are set to 0 to omit these base linear networks. 

It is noticed that the encoder $\mathbf{W}$ can be a pre-trained building block based on $\mathcal{D^{S}}$. In our method, on the one hand, it can be used to construct the sample-wise indicators, and meanwhile, on the other hand, we can also use the latent embedding $\mathbf{E}=\mathbf{W}\mathbf{X}$ as the initial features of the composite linear networks of Eq.~\ref{fx}. 

% In practice, we can reformulate Eq.~\ref{encoder} with the tied weights~\cite{ranzato2007sparse}, i.e., ${\mathbf{W}}' = \mathbf{W}^{\mathsf{T}}$,
% \begin{equation}
% \mathop{\arg \min}_{\mathbf{W}} \ \left \| \mathbf{X} - \mathbf{W}^{\mathsf{T}} \mathbf{W}\mathbf{X} \right \|_{2},
% \end{equation}
%determine the impact of each $\mathbf{E}_{x}^{(i)}$, we calculate the variance of each $\mathbf{E}_{x}^{(i)}$
%$\mathrm{Var} \left ( \mathbf{E}_{x}^{(i)} \mid \mu_{E_{x}}\right )$
%design a switch function $s$
% \begin{equation}
% \begin{split}
% &\frac{\partial L}{\partial \mathbf{\Theta}}=0 \Rightarrow \mathbf{\Theta} = \sum_{l=1}^{N} (\alpha_{l} - \alpha_{l}^{*})\mathbf{\Phi}(x_{l}) \\
% &\frac{\partial L}{\partial \gamma_{l}}=0 \Rightarrow C = \alpha_{l} + \mu_{l} \\
% &\frac{\partial L}{\partial \gamma_{l}^{*}}=0 \Rightarrow C = \alpha_{l}^{*} + \mu_{l}^{*}.
% \end{split}
% \end{equation}

\subsection{Orthogonality and Low-Rankness}
The proposed composite linear networks provide the guarantee of a parsimonious and on-device-friendly framework for our \textit{ParsNets}. 
In this section, inspired by subspace transformation~\cite{qiu2015learning}, we enforce a maximal margin geometry on the mapped features, i.e., low-rank structure to intra-class features and high-rank structure to inter-class features, respectively, which further guarantees the model’s generalization ability.
%we use the nuclear norm of the mapped features via each linear network as the convex surrogate of rank function, and enforce a maximal margin geometry to further guarantee the model’s generalization ability. 
%

Concretely, without loss of generality, given the weight $\Theta_{i}$ of a linear network described in Eq.~\ref{fx}, we rewrite it as the matrix form as $\Theta_{i}^\mathsf{T} \mathbf{X}$, where $\mathbf{X} = \left [ x_{1} \mid x_{2} \mid \dots \mid x_{N} \right ] \in \mathcal{D^{S}}$ with each column $x_{i} \in \mathbb{R}^{d}$ denoting a labeled sample from total $\left | \mathcal{Y^{S}} \right |$ seen classes. Let $\mathbf{X}_{v}$ denote the sample matrix extracted from the columns of $\mathbf{X}$ that belong to $v$-th class, we construct a minimization problem as:
% \begin{equation}
% \mathop{\arg \min}_{\Theta_{i}} \ \sum_{v=1}^{\left | \mathcal{Y^{S}} \right |} \left \| \Theta_{i}^\mathsf{T} \mathbf{X}_{v} \right \|_{*}  - \left \| \Theta_{i}^\mathsf{T} \mathbf{X} \right \|_{*}, \ \mathrm{s.t.} \ \left \| \Theta_{i}^\mathsf{T} \right \|_{2} = 1,   
% \end{equation}
\begin{small}
\begin{equation}
\begin{split}
&\mathop{\arg \min}_{\Theta_{i}} \ \sum_{i=1}^{K} \sum_{v=1}^{\left | \mathcal{Y^{S}} \right |} \left \| \Theta_{i}^\mathsf{T} \mathbf{X}_{v} \right \|_{*}  - \left \| \Theta_{i}^\mathsf{T} \mathbf{X} \right \|_{*}, \\
&\mathrm{s.t.}
% \left\{\begin{matrix}
% \left \| \Theta_{i}^\mathsf{T} \right \| = 1 \\
% \left \langle \Theta_{i}, \Theta_{j} \right \rangle = 0
% \end{matrix}\right.,
%
\left \| \Theta_{i}^\mathsf{T} \right \| = 1, \left \langle \Theta_{i}, \Theta_{j} \right \rangle = 0 (\forall j \in [1,K], i\ne j),
\end{split}
\label{nuclear}
\end{equation}
\end{small}
where $\left \| \cdot \right \|_{*}$ is the nuclear norm which is a relaxation form of the non-differentiable rank function $\mathrm{rank}(\cdot)$, i.e., $\left \| \Theta_{i}^\mathsf{T} \mathbf{X}_{v} \right \|_{*}$, $\left \| \Theta_{i}^\mathsf{T} \mathbf{X} \right \|_{*}$ $\approx$ $\mathrm{rank}(\Theta_{i}^\mathsf{T} \mathbf{X}_{v})$, $\mathrm{rank}(\Theta_{i}^\mathsf{T} \mathbf{X})$. 

Notably, in Eq.~\ref{nuclear}, $\left \| \Theta_{i}^\mathsf{T} \mathbf{X}_{v} \right \|_{*}$ minimizes the rank of the mapped feature matrix of each class, which can encourage intra-class samples to reside in the same linear subspace. Meanwhile, $\left \| \Theta_{i}^\mathsf{T} \mathbf{X} \right \|_{*}$ maximizes the rank of the mapped feature matrix of all classes, which can additionally encourage the aforementioned linear subspaces to be orthogonal from each other, thus maximizing the generalization ability.  
Moreover, $\left \| \Theta_{i}^\mathsf{T} \right \| = 1$ is an extra regularization term, i.e., corresponds to $\varphi \left ( \mathbf{\Theta} \right )$ in Eq.~\ref{vs-mapping}, to avoid zero solution $\Theta_{i}^\mathsf{T} = \mathbf{0}$, and $\left \langle \Theta_{i}, \Theta_{j} \right \rangle = 0$ pushes each linear network to be independent of each other. 

Now, we prove that the global minimum of Eq.~\ref{nuclear} can be reached as 0, when each $\mathbf{X}_{v}$ is orthogonal to the other.

\begin{proposition}
If $\mathbf{X}_{v}$ and $\mathbf{X}_{{v}'}$ are orthogonal to each other, where $\forall v,{v}' \in [1,K]$ and $v \ne {v}'$, then Eq.~\ref{nuclear} reaches the global minimum, i.e., Eq.~\ref{nuclear}=0.
\label{proposition1}
\end{proposition}

To prove Proposition~\ref{proposition1}, we first present two theorems, i.e., Theorem~\ref{theorem1} and Theorem~\ref{theorem2}.

\begin{theorem}
Let $\mathbf{M}$ and $\mathbf{N}$ be matrices that have the same row dimensions, and let $\left [ \mathbf{M}, \mathbf{N} \right ] $ be the concatenation of $\mathbf{M}$ and $\mathbf{N}$, we have:
\begin{equation}
\left \| \left [ \mathbf{M}, \mathbf{N} \right ] \right \|_{*} \le \left \| \mathbf{M} \right \|_{*} + \left \| \mathbf{N} \right \|_{*}.
\end{equation}
\label{theorem1}
\end{theorem}

\begin{proof}[Proof of \textbf{Theorem~\ref{theorem1}}]
Can be proved easily via:
\begin{equation}
\begin{split}
\left \| \mathbf{M} \right \|_{*} + \left \| \mathbf{N} \right \|_{*} &= \left \| [\mathbf{M} \ \mathbf{0}] \right \|_{*} + \left \| [\mathbf{0} \ \mathbf{N}] \right \|_{*} \\
&\ge \left \| [\mathbf{M} \ \mathbf{0}] + [\mathbf{0} \ \mathbf{N}] \right \|_{*} 
= \left \| [\mathbf{M}, \mathbf{N}] \right \|_{*}. 
\end{split}
\end{equation}
\end{proof}

\begin{theorem}
Let $\mathbf{M}$ and $\mathbf{N}$ be matrices that have the same row dimensions, and let $\left [ \mathbf{M}, \mathbf{N} \right ] $ be the concatenation of $\mathbf{M}$ and $\mathbf{N}$, we have:
\begin{equation}
\left \| \left [ \mathbf{M}, \mathbf{N} \right ] \right \|_{*} = \left \| \mathbf{M} \right \|_{*} + \left \| \mathbf{N} \right \|_{*},
\label{tm2}
\end{equation}
when $\mathbf{M}$ and $\mathbf{N}$ are column-wise orthogonal.
\label{theorem2}
\end{theorem}

\begin{proof}[Proof of \textbf{Theorem~\ref{theorem2}}]
We apply the singular value decomposition to $\mathbf{M}$ and $\mathbf{N}$ as:
\begin{equation}
\begin{split}
&\mathbf{M} = \left [ \mathbf{U}_{\mathbf{M}1} \mathbf{U}_{\mathbf{M}2} \right  ] \begin{bmatrix}
\sum_{\mathbf{M}} & 0 \\
 0 & 0
\end{bmatrix}
{\left [ \mathbf{U}_{\mathbf{M}1} \mathbf{U}_{\mathbf{M}2} \right  ]}', \\
&\mathbf{N} = \left [ \mathbf{U}_{\mathbf{N}1} \mathbf{U}_{\mathbf{N}2} \right  ] \begin{bmatrix}
\sum_{\mathbf{N}} & 0 \\
 0 & 0
\end{bmatrix}
{\left [ \mathbf{U}_{\mathbf{N}1} \mathbf{U}_{\mathbf{N}2} \right  ]}', 
\end{split}
\end{equation}
where $\sum_{\mathbf{M}}$ and $\sum_{\mathbf{N}}$ contain non-zero singular values, then we can have:
\begin{equation}
\begin{split}
&\mathbf{M}{\mathbf{M}}' = \left [ \mathbf{U}_{\mathbf{M}1} \mathbf{U}_{\mathbf{M}2} \right  ] \begin{bmatrix}
{\sum_{\mathbf{M}}}^{2} & 0 \\
 0 & 0
\end{bmatrix}
{\left [ \mathbf{U}_{\mathbf{M}1} \mathbf{U}_{\mathbf{M}2} \right  ]}', \\
&\mathbf{N}{\mathbf{N}}' = \left [ \mathbf{U}_{\mathbf{N}1} \mathbf{U}_{\mathbf{N}2} \right  ] \begin{bmatrix}
{\sum_{\mathbf{N}}}^{2} & 0 \\
 0 & 0
\end{bmatrix}
{\left [ \mathbf{U}_{\mathbf{N}1} \mathbf{U}_{\mathbf{N}2} \right  ]}'.
\end{split}
\label{pf2-2}
\end{equation}
Given that $\mathbf{M}$ and $\mathbf{N}$ are column-wise orthogonal, i.e., ${\mathbf{U}_{\mathbf{M}1}}' \mathbf{U}_{\mathbf{N}1} = 0$, then Eq.~\ref{pf2-2} can be rewritten as:
\begin{equation}
\begin{split}
&\mathbf{M}{\mathbf{M}}' = \left [ \mathbf{U}_{\mathbf{M}1} \mathbf{U}_{\mathbf{N}1} \right  ] \begin{bmatrix}
{\sum_{\mathbf{M}}}^{2} & 0 \\
 0 & 0
\end{bmatrix}
{\left [ \mathbf{U}_{\mathbf{M}1} \mathbf{U}_{\mathbf{N}1} \right  ]}', \\
&\mathbf{N}{\mathbf{N}}' = \left [ \mathbf{U}_{\mathbf{M}1} \mathbf{U}_{\mathbf{N}1} \right  ] \begin{bmatrix}
0 & 0 \\
 0 & {\sum_{\mathbf{N}}}^{2}
\end{bmatrix}
{\left [ \mathbf{U}_{\mathbf{M}1} \mathbf{U}_{\mathbf{N}1} \right  ]}'.
\end{split}
\end{equation}
Then we can have:
\begin{equation}
\begin{split}
&[\mathbf{M}, \mathbf{N}] {[\mathbf{M}, \mathbf{N}]}' = \mathbf{M}{\mathbf{M}}' + \mathbf{N}{\mathbf{N}}' \\
&= \left [ \mathbf{U}_{\mathbf{M}1} \mathbf{U}_{\mathbf{N}1} \right  ] \begin{bmatrix}
{\sum_{\mathbf{M}}}^{2} & 0 \\
 0 & {\sum_{\mathbf{N}}}^{2}
\end{bmatrix}
{\left [ \mathbf{U}_{\mathbf{M}1} \mathbf{U}_{\mathbf{N}1} \right  ]}'.
\end{split}
\end{equation}
Since the nuclear norm $\left \| \mathbf{M} \right \|_{*}$ equals to the sum of the square root of the singular values of $\mathbf{M}{\mathbf{M}}'$, so we can have $\left \| \left [ \mathbf{M}, \mathbf{N} \right ] \right \|_{*} = \left \| \mathbf{M} \right \|_{*} + \left \| \mathbf{N} \right \|_{*}$ that proves Eq.~\ref{tm2}.
\end{proof}

\begin{proof}[Proof of \textbf{Proposition~\ref{proposition1}}]
It is obvious that Theorem~\ref{theorem1} and Theorem~\ref{theorem2} can be extended to multiple matrices. As a result, for Eq.~\ref{nuclear}, we have:
\begin{equation}
\sum_{i=1}^{K} \sum_{v=1}^{\left | \mathcal{Y^{S}} \right |} \left \| \Theta_{i}^\mathsf{T} \mathbf{X}_{v} \right \|_{*}  - \left \| \Theta_{i}^\mathsf{T} \mathbf{X} \right \|_{*} \ge 0.
\label{pf3}
\end{equation}
Based on Theorem~\ref{theorem2} and Eq.~\ref{pf3}, the minimization problem described in Eq.~\ref{nuclear} can achieve the global minimum of 0, if the column spaces of all pairs of matrices are orthogonal. 
\end{proof}

% \subsection{Unified Framework}
% \begin{small}
% \begin{algorithm}[tb]
% \caption{Unified Framework of \textit{ParsNets}}
% \label{algo:pFedKM}
% \textbf{Input}: ${\Theta}_I^0, {\Omega}_K^0, P^0, T, R, S, K, \lambda, \eta, \alpha, \beta$. \\
% \textbf{Output}: ${\Theta}_I^T$
% \begin{algorithmic}[1]
%   \FOR {$t=0$ to $T-1$} 
%     \STATE Server sends ${\Omega}_K^t$ to clients according to $P^t$.
%     \FOR {local device $i=1$ to $N$ in parallel} 
%       \STATE Initialization: ${\Omega}_{I,0}^t = {\Omega}_I^t$.
%       \STATE Local update for the sub-problem of $G({\Theta}_I, {\Omega}_K)$: 
%       \FOR {$r=0$ to $R-1$} 
%         \FOR {$s=0$ to $S-1$}
%             \STATE Update the personalized model:
%             \STATE ${\theta}_i^{s+1} = {\theta}_i^s - \eta \nabla F_i({\theta}_i^s)$.
%         \ENDFOR
%         \STATE Local update: 
%         \STATE ${\omega}_{i,r+1}^t = {\omega}_{i,r}^t-\beta\nabla G({\omega}_{i,r}^t)$.
%       \ENDFOR
%     \ENDFOR
%     \STATE Clients send back ${\omega}_{i,R}^t$ to the server.
%     \STATE Server conducts (\textit{k}-means) clustering on models ${\Omega}_{I,R}^t$ to obtain $P^{t+1}$.
%     \STATE Global aggregation: 
%     \STATE ${\Omega}_K^{t+1} = {\Omega}_K^t - \alpha({\Omega}_K^t - {\Omega}_{I,R}^t P^{t+1})$.
%   \ENDFOR
%   \STATE {\textbf{return} The personalized models ${\Theta}_I$.}
% \end{algorithmic}
% \end{algorithm}
% \end{small}

\begin{table*}[htbp]
\setlength{\tabcolsep}{4mm}{
\fontsize{7.5}{10}\selectfont
    \begin{center}
            \begin{tabular}{|l|c|c|c|c|c|}        
                \hline                   
\textbf{Method} &\textbf{Venue} & \textbf{AWA2} & \textbf{CUB-200} & \textbf{SUN} & \textbf{aPY}   \\
\hline            
SE-ZSL~\cite{kumar2018generalized} &CVPR~$'{18}$ &80.8 &60.3 &64.5 &39.8 \\
%GAL \cite{yu2019zero}  &?~$'{?}$ & & & & \\
f-CLSWGAN~\cite{xian2018feature} &CVPR~$'{18}$ &68.2 &57.3 &60.8 &- \\
Zhu~\textit{et al.}~\cite{zhu2019semantic}  &NeurIPS~$'{19}$ &\underline{\textbf{83.5}} &70.5 &- &- \\
AREN~\cite{xie2019attentive} &CVPR~$'{19}$ &67.9 &70.7 &61.7 &44.1 \\
APNet~\cite{liu2020attribute} &AAAI~$'{20}$ &68.0 &57.7 &62.3 &41.3 \\
RGEN~\cite{xie2020region} &ECCV~$'{20}$ &73.6 &76.1 &63.8 &44.4 \\
OCD-CVAE~\cite{keshari2020generalized} &CVPR~$'{20}$ &\underline{81.7} &60.8 &\textbf{68.9} &- \\
DAZLE~\cite{huynh2020fine} &CVPR~$'{20}$ &75.2 &64.1 &62.5 &- \\
LsrGAN~\cite{vyas2020leveraging} &ECCV~$'{20}$ &66.4 &60.3 &62.5 &- \\
APN~\cite{xu2020attribute} &NeurIPS~$'{20}$ &68.4 &72.0 &61.6 &- \\
HSVA~\cite{chen2021hsva}  &NeurIPS~$'{21}$ &70.6 &62.8 &63.8 &- \\
VGSE~\cite{xu2022vgse}   &CVPR~$'{22}$ &64.0 &28.9 &38.1 &- \\
TDCSS~\cite{feng2022non}    &CVPR~$'{22}$ &71.2 &61.1 &- &- \\
PSVMA~\cite{liu2023progressive} &CVPR~$'{23}$ &79.4 &72.9 &\underline{66.5} &\underline{45.9} \\
GKU~\cite{guo2023graph} &AAAI~$'{23}$ &- &\underline{76.9} &- &- \\
DGZ~\cite{chen2023deconstructed} &AAAI~$'{23}$ &74.0 &\textbf{80.1} &65.4 &\textbf{46.6} \\
\hline
\textbf{\textit{ParsNets} (ours)} &Proposed &\textbf{82.6} &\underline{\textbf{80.3}} &\underline{\textbf{70.2}} &\underline{\textbf{48.7}} \\
\hline
\end{tabular}
\caption{Comparison of ZSL performance with state-of-the-art competitors (accuracy \%). The best result is marked in `\underline{\textbf{underlined bold}}', the second in `\textbf{bold}', and the third in `\underline{underlined}'. `-' indicates there is no reported result/open source or not applicable to the dataset.}\label{Results-ZSL}
\end{center}
}
%\vspace{-4mm}
\end{table*}

\begin{table*}[t]
\fontsize{7.5}{10}\selectfont
    \begin{center}
    \begin{tabular}{|l|c|cc|c|cc|c|cc|c|cc|c|}        
\hline
\multicolumn{1}{|c|}{\multirow{2}{*}{\textbf{Method}}} 
& \multicolumn{1}{c|}{\multirow{2}{*}{\textbf{Venue}}} 
%& \multicolumn{1}{|c|}{\multirow{2}{*}{Fine-grained}} 
& \multicolumn{3}{c|}{\textbf{AWA2}}                                              
& \multicolumn{3}{c|}{\textbf{CUB-200}}
& \multicolumn{3}{c|}{\textbf{SUN}}    
& \multicolumn{3}{c|}{\textbf{aPY}} \\ 
\cline{3-14} 
\multicolumn{1}{|c|}{}                   
%& \multicolumn{1}{c|}{}                   
& \multicolumn{1}{c|}{}                   
& \multicolumn{1}{c|}{\textbf{U}} & \multicolumn{1}{c|}{\textbf{S}} & \multicolumn{1}{c|}{\textbf{H}} 
& \multicolumn{1}{c|}{\textbf{U}} & \multicolumn{1}{c|}{\textbf{S}} & \multicolumn{1}{c|}{\textbf{H}}
& \multicolumn{1}{c|}{\textbf{U}} & \multicolumn{1}{c|}{\textbf{S}} & \multicolumn{1}{c|}{\textbf{H}} 
& \multicolumn{1}{c|}{\textbf{U}} & \multicolumn{1}{c|}{\textbf{S}} & \multicolumn{1}{c|}{\textbf{H}} \\
\hline
f-CLSWGAN~\cite{xian2018feature} &CVPR~$'{18}$ &57.9 &61.4 &59.6 &43.7 &57.7 &49.7 &42.6 &36.6 &39.4 &- &- &- \\
SE-GZSL~\cite{kumar2018generalized} &CVPR~$'{18}$ &58.3 &68.1 &62.8 &41.5 &53.3 &46.7 &40.9 &30.5 &34.9 &- &- &- \\
Zhu~\textit{et al.}~\cite{zhu2019semantic} &NeurIPS~$'{19}$ &37.6 &\underline{\textbf{87.1}} &52.5 &36.7 &71.3 &48.5 &- &- &- &- &- &- \\
AREN~\cite{xie2019attentive} &CVPR~$'{19}$ &54.7 &79.1 &64.7 &63.2 &69.0 &66.0 &40.3 &32.3 &35.9 &30.0 &47.9 &36.9 \\
LsrGAN~\cite{vyas2020leveraging} &ECCV~$'{20}$ &54.6 &74.6 &63.0 &48.1 &59.1 &53.0 &44.8 &37.7 &40.9 &- &- &- \\
DAZLE~\cite{huynh2020fine} &CVPR~$'{20}$ &\underline{75.7} &60.3 &67.1 &59.6 &56.7 &58.1 &24.3 &\underline{\textbf{52.3}} &33.2 &- &- &- \\
OCD-CVAE~\cite{keshari2020generalized} &CVPR~$'{20}$ &59.5 &73.4 &65.7 &44.8 &59.9 &51.3 &44.8 &42.9 &\underline{43.8} &- &- &- \\
RGEN~\cite{xie2020region} &ECCV~$'{20}$ &67.1 &76.5 &\underline{71.5} &60.0 &\underline{73.5} &66.1 &44.0 &31.7 &36.8 &30.4 &\underline{48.1} &\underline{37.2} \\
APNet~\cite{liu2020attribute} &AAAI~$'{20}$ &\underline{\textbf{83.9}} &54.8 &66.4 &55.9 &48.1 &51.7 &40.6 &35.4 &37.8 &\underline{\textbf{74.7}} &32.7 &45.5 \\
HSVA~\cite{chen2021hsva}  &NeurIPS~$'{21}$ &56.7 &79.8 &66.3 &52.7 &58.3 &55.3 &\underline{48.6} &39.0 &43.3 &- &- &- \\
TDCSS~\cite{feng2022non} &CVPR~$'{22}$ &59.2 &74.9 &66.1 &44.2 &62.8 &51.9 &- &- &- &- &- &- \\
VGSE~\cite{xu2022vgse}   &CVPR~$'{22}$ &51.2 &\textbf{81.8} &63.0 &21.9 &45.5 &29.5 &24.1 &31.8 &27.4 &- &- &- \\
PSVMA~\cite{liu2023progressive} &CVPR~$'{23}$ &73.6 &77.3 &\textbf{75.4} &\underline{70.1} &\textbf{77.8} &\textbf{73.8} &\underline{\textbf{61.7}} &\underline{45.3} &\textbf{52.3} &- &- &- \\
GKU~\cite{guo2023graph} &AAAI~$'{23}$ &- &- &- &52.3 &71.1 &60.3 &- &- &- &- &- &- \\
DGZ~\cite{chen2023deconstructed} &AAAI~$'{23}$ &65.9 &78.2 &\underline{71.5} &\textbf{71.4} &64.8 &\underline{68.0} &49.9 &37.6 &42.8 &\underline{38.0} &\textbf{63.5} &\textbf{47.6} \\
% f-CLSWGAN~\cite{xian2018feature} &CVPR~$'{18}$ &57.9 &61.4 &59.6 &43.7 &57.7 &49.7 &42.6 &36.6 &39.4 &- &- &- \\
% SE-GZSL~\cite{kumar2018generalized} &CVPR~$'{18}$ &58.3 &68.1 &62.8 &41.5 &53.3 &46.7 &40.9 &30.5 &34.9 &- &- &- \\
\hline
\textbf{\textit{ParsNets} (ours)} &Proposed &\textbf{77.6} &\underline{81.4} &\underline{\textbf{79.5}} &\underline{\textbf{72.8}} &\underline{\textbf{79.4}} &\underline{\textbf{76.0}} &\textbf{57.2} &\textbf{49.5} &\underline{\textbf{53.1}} &\textbf{42.3} &\underline{\textbf{68.6}} &\underline{\textbf{52.3}} \\
\hline
\end{tabular}
\caption{Comparison of GZSL performance with state-of-the-art competitors (accuracy \%). The best result is marked in `\underline{\textbf{underlined bold}}', the second in `\textbf{bold}', and the third in `\underline{underlined}'. `-' indicates there is no reported result/open source or not applicable to the dataset.}\label{Results-GZSL}
\end{center}
%}
%\vspace{-4mm}
\end{table*}

\section{Experiments}
\subsection{Experimental Setup}
\noindent \textbf{Dataset.} 
%\subsubsection{Dataset}
We evaluate our \textit{ParsNets} on four widely used ZSL/GZSL benchmark datasets including:
\begin{itemize}
    \item \textbf{AWA2}: The Animals with Attributes 2~\cite{xian2018zero} is the most widely used ZSL/GZSL benchmark that contains 37,322 samples of 50 animal classes, each with 85 numeric attribute values as the class-level descriptors. 
    \item \textbf{CUB-200}: The Caltech-UCSD Birds-200-2011~\cite{WahCUB_200_2011} is another important benchmark that consists of 11,788 samples of 200 bird species classes, each with 312 attributes as the class-level descriptors. Besides, it also provides fine-grained annotations for fine-grained ZSL/GZSL methods.
    \item \textbf{SUN}: The SUN Attribute~\cite{patterson2014sun} is a large-scale scene image dataset that contains 14,340 samples of 717 scenario-style classes, each with 102 attributes as the class-level descriptors.
    \item \textbf{aPY}: The Attribute Pascal and Yahoo~\cite{farhadi2009describing} is collected from Pascal VOC 2008 and Yahoo that consists of 15,339 samples of 32 classes, each with 64 attributes as the class-level descriptors.
\end{itemize}

% \textbf{AWA2}: Animals with Attributes 2~\cite{xian2018zero} that contains 37,322 samples of 50 animal classes, each with 85 numeric attribute values as the class-level descriptors; 
% %
% \textbf{CUB}: Caltech-UCSD Birds-200-2011~\cite{WahCUB_200_2011} that consists of 11,788 samples of 200 bird classes, each with 312 attributes as the class-level descriptors; 
% %
% \textbf{SUN}: SUN Attribute~\cite{patterson2014sun} that contains 14,340 samples of 717 scenario-style classes, each with 102 attributes as the class-level descriptors; and 
% %
% \textbf{APY}: Attribute Pascal and Yahoo~\cite{farhadi2009describing} that consists of 15,339 samples of 32 classes, each with 64 attributes as the class-level descriptors.
%

As to the splitting strategy of seen and unseen classes for each dataset, we and all involved competitors strictly follow \cite{xian2018zero}, which is the most adopted benchmark splitting for ZSL/GZSL, to ensure a fair comparison.

\noindent \textbf{Evaluation Metrics.} 
%\subsubsection{Evaluation Metrics}
Two different scenarios are considered in our experiments, including the classic ZSL and GZSL. 
For ZSL, the recognition only searches the test samples from unseen classes and reports the multi-way classification accuracy as in previous works for our method and each involved competitor. 
Differently, for GZSL, we compute the average per-class prediction accuracy on test samples from unseen classes (U) and seen classes (S), respectively, and report the Harmonic Mean calculated by $H = \left ( 2 \times U \times S \right )/\left ( U + S \right )$ to quantify the aggregate performance across both seen and unseen classes.
As to the competitors, we select representative \textbf{deep learning-based ZSL/GZSL methods} based on the following criteria: 1) formally published in the most recent years; 2) covered a wide range of models; 3) all of them clearly represented the state-of-the-art.

\noindent \textbf{Implementation.} 
%\subsubsection{Implementation}
We implement the proposed \textit{ParsNets} on Raspberry Pi 4B\footnote{www.raspberrypi.com/products/raspberry-pi-4-model-b}, which is a widely used low-cost edge device platform equipped with ARM Cortex-A72 CPU and 4GB RAM. Similar to the server-based computing architecture, the edge platform is installed with Ubuntu 20.10, Miniconda3, and PyTorch 1.8.0. that can support most modeling frameworks. 
To construct the \textit{ParsNets}, we use the single-layer neural network with ReLU activation function for each of the base linear networks in Eq.~\ref{fx}, and the total number of the networks $K$ is empirically set as 200 for all datasets. 
As to the sample-wise indicators in Eq.~\ref{final-indicators}, we rank all variances and then enable $k$ ranges in \{10, 20, 30, 40, 80, 120, 160, 200\}, which corresponds to \{5\%, 10\%, 15\%, 20\%, 40\%, 60\%, 80\%, 100\%\} of the total base linear networks that have been activated during training. 
As to the visual representation, to align with the common practice of most existing methods, we use the 2048-dimensional visual features extracted from ResNet for each input sample. 
It is noticed that the other competitors are all implemented and running on server-based computers along with powerful GPUs and large storage, further highlighting the on-device-friendly functionality of our method.

% \begin{figure}[t]
%   \centering
% \centerline{\includegraphics[width=0.35\textwidth]{Raspberry-Pi-Set.jpg}}
%   \caption{A.}
%   \label{fig1}
% %\vspace{-5mm}
% \end{figure}

\subsection{Comparison of ZSL Performance}
We compare the proposed \textit{ParsNets} with 16 state-of-the-art deep learning-based competitors in the classic ZSL scenario and report the multi-way classification accuracy in Table~\ref{Results-ZSL}. It can be observed from the results that our method outperforms most deep learning-based competitors on all datasets. 
For example, DGZ~\cite{chen2023deconstructed} and TDCSS~\cite{feng2022non} are the most two powerful generative model-based ZSL competitors that can achieve 80.1\% and 61.1\% recognition accuracy on CUB-200. In contrast, despite as a non-deep method, our method obtains 0.2\% and 19.2\% higher performance with much lower computing cost. 
On the other hand, as a complex graph fine-grained method that utilized powerful GNNs as the sample representation, GKU~\cite{guo2023graph} achieves 76.9\% on CUB-200. In contrast, our method obtains a much better performance of 80.3\% with/on only the resource-constrained device.

\begin{figure*}[t]
\centering
\subfigure[AWA2-Raw]{\label{m3-original}
\includegraphics[width=1.6in]{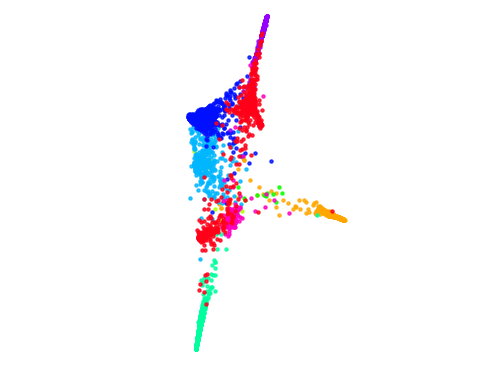}}
\subfigure[AWA2-Mapped]{\label{m3-mapped}
\includegraphics[width=1.6in]{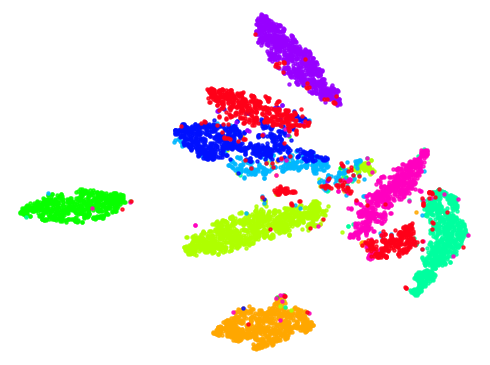}}
\subfigure[CUB-200-Raw]{\label{m-original}
\includegraphics[width=1.6in]{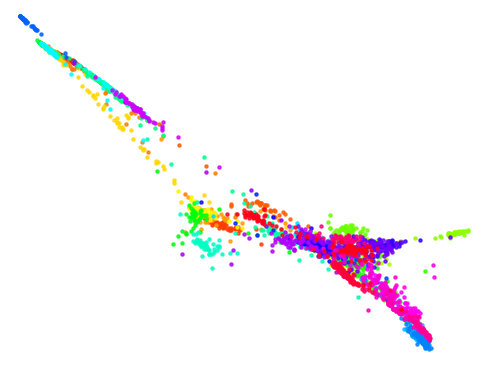}}
\subfigure[CUB-200-Mapped]{\label{m-mapped}
\includegraphics[width=1.6in]{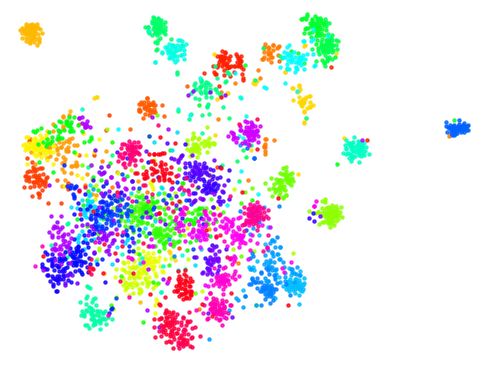}}
\caption{Visualization results of mapping robustness: (a) raw features of unseen classes in AWA2, (b) mapped features of unseen classes in AWA2, (c) raw features of unseen classes in CUB-200, and (d) mapped features of unseen classes in CUB-200 (better viewed in color).} 
\label{tsne}
%\vspace{-5.5mm}
\end{figure*}

\subsection{Comparison of GZSL Performance}
As shown in Table~\ref{Results-GZSL}, we compare the proposed \textit{ParsNets} with 15 state-of-the-art deep learning-based competitors in the GZSL scenario and report the average per-class prediction accuracy on unseen classes (U), seen classes (S), and their Harmonic Mean (H). 
We can observe that our method also constantly outperforms other deep learning-based competitors by even improved margins than that of the ZSL scenario on all datasets. 
Specifically, in GZSL, most competitors can hardly achieve balanced performance in both seen and unseen classes due to the domain biased over/under-fittings problem. For example, in Zhu~\textit{et al.}~\cite{zhu2019semantic} and APNet~\cite{liu2020attribute}, there exist 49.5\% and 42.0\% margins between the accuracy of seen and unseen classes in AWA2 and aPY, respectively, thus the overall performance, i.e, Harmonic Mean, is significantly poor for real-world application. Moreover, even some of the most powerful competitors such as PSVMA~\cite{liu2023progressive} and DGZ~\cite{chen2023deconstructed} can still have a nonnegligible margin, i.e., 16.4\% and 12.3\% in SUN and AWA2, respectively. 
In contrast, due to the utilization of the proposed sample-wise composite semantic predictor and the constructed maximal margin geometry, our method can significantly relieve the domain biased over/under-fittings problem and obtains a more balanced performance in both seen and unseen classes.

% \begin{table}[htbp]
% \fontsize{9}{10}\selectfont
% \centering
% %\setlength{\tabcolsep}{1.4mm}{
% \begin{threeparttable}     
% \begin{tabular}{cccccccc}  
% \toprule  
% \multicolumn{2}{c}{\textbf{Module}}
% &\multicolumn{1}{c}{\textbf{AWA3}}
% &\multicolumn{1}{c}{\textbf{CUB-200}}
% &\multicolumn{1}{c}{\textbf{SUN}}
% &\multicolumn{1}{c}{\textbf{aPY}}\cr  
% \cmidrule(lr){1-2} 
% %\cmidrule(lr){5-5} \multirow{2}{*} \multicolumn{1}{|c|}{\multirow{2}{*}{\textbf{Method}}} 
% %\cmidrule(lr){6-6} 
% \textbf{SCLN}    &\textbf{MMG} &(\%)  &(\%) &(\%)  &(\%)  \cr
% \midrule  
% $\checkmark$   &              &67.2  &58.8 &51.4  &38.3 \cr
%                &$\checkmark$  &69.5  &67.2 &55.2  &40.1 \cr
% $\checkmark$   &$\checkmark$  &82.6  &80.3 &70.2  &48.7 \cr
% \bottomrule  
% \end{tabular}
% \caption{Ablation study}\label{ablation_study}  
% \end{threeparttable}
% %}
% \end{table}

% \begin{figure}[htbp]
%   \centering
% \centerline{\includegraphics[width=0.49\textwidth]{sparse.pdf}}
%   \caption{Sparseness analysis on AWA2, CUB-200, SUN, and aPY datasets. The optimal sparseness is 15\%, 20\%, 40\%, and 15\%, respectively (better viewed in color).}
%   \label{sparse}
% %\vspace{-5mm}
% \end{figure}

% \begin{figure}[h]
% %\vspace{-3.5mm}
% \centering
% \subfigure[2D Subspace Map]{\label{m3-original}
% \includegraphics[width=1.6in]{3class_2d.png}}
% %
% \subfigure[3D Subspace Map]{\label{m3-mapped}
% \includegraphics[width=1.6in]{3class_3d.png}}
% %\vspace{-3mm}
% \caption{Subspace Visualization.} 
% %\vspace{-3.5mm}
% \end{figure}

\subsection{Mapping Robustness}
To further demonstrate the effectiveness of our method in relieving the domain biased over/under-fittings problems, we visualize the raw and mapped features of samples from test unseen classes of AWA2 (10 unseens) and CUB-200 (50 unseens), respectively, using t-SNE~\cite{van2008visualizing} in Figure~\ref{tsne}. 
It can be observed that we can pose more subspaces between different classes with the obtained mapped features than the results with raw features. 
Specifically, Figure~\ref{m3-original} and Figure~\ref{m-original} are the visualizations of raw features of AWA2 and CUB-200, respectively, where most classes are clustered in panhandle subspaces. 
In contrast, Figure~\ref{m3-mapped} and Figure~\ref{m-mapped} are the visualizations of mapped features of AWA2 and CUB-200, respectively, by reusing the trained model based on only seen classes. 
It is obvious that the mapped feature space results in more subspaces for different classes, thus our method can be more separable and robust across both seen and unseen classes. 
%and meanwhile, significantly relieve the domain bias problem in ZSL/GZSL.
%is much more separable than the raw feature space, thus our method can be more robust to the gap between seen and unseen classes, and hence significantly relieve the domain bias problem in ZSL/GZSL.
%where we can see that the mapped features are 
%three randomly selected classes using their raw and mapped features, respectively, where we can see that the mapped features are nearly orthogonal from each other while the raw features are quite clustered. Figure~\ref{m-original} and Figure~\ref{m-mapped} visualize the results of all classes and we can observe that the mapped feature space is much more separable than the raw feature space. 
%
%Therefore, our method can be more robust to the gap between seen and unseen classes, and hence significantly relieve the domain bias problem in ZSL/GZSL.

\begin{table}[htbp]
\fontsize{9}{8}\selectfont
\centering
\begin{threeparttable}     
\begin{tabular}{cccccccc}  
\toprule  
\multicolumn{2}{c}{\textbf{Module}}
&\multicolumn{1}{c}{\textbf{AWA3}}
&\multicolumn{1}{c}{\textbf{CUB-200}}
&\multicolumn{1}{c}{\textbf{SUN}}
&\multicolumn{1}{c}{\textbf{aPY}}\cr  
\cmidrule(lr){1-2} 
%\cmidrule(lr){5-5} \multirow{2}{*} \multicolumn{1}{|c|}{\multirow{2}{*}{\textbf{Method}}} 
%\cmidrule(lr){6-6} 
\textbf{SCLN}    &\textbf{MMG} &(\%)  &(\%) &(\%)  &(\%)  \cr
\midrule  
$\checkmark$   &              &67.2  &58.8 &51.4  &38.3 \cr
               &$\checkmark$  &69.5  &67.2 &55.2  &40.1 \cr
$\checkmark$   &$\checkmark$  &82.6  &80.3 &70.2  &48.7 \cr
\bottomrule  
\end{tabular}
\caption{Ablation study}\label{ablation_study}  
\end{threeparttable}
%}
%\vspace{-5.5mm}
\end{table}

\begin{figure}[htbp]
  \centering
\centerline{\includegraphics[width=0.4\textwidth]{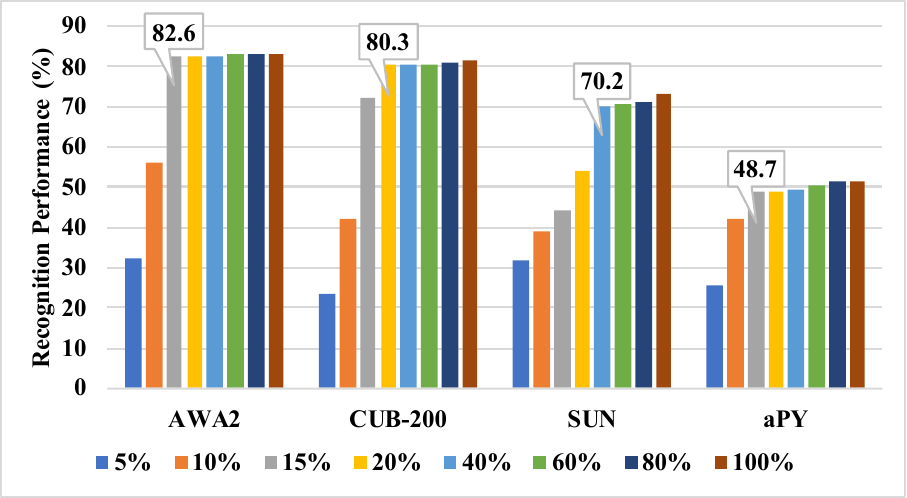}}
  \caption{Sparseness analysis on AWA2, CUB-200, SUN, and aPY datasets. The optimal sparseness is 15\%, 20\%, 40\%, and 15\%, respectively (better viewed in color).}
  \label{sparse}
%\vspace{-5.5mm}
\end{figure}

\subsection{Further Analysis}
We design two additional experiments to further analyze the impact of some key building blocks of our method. 

\noindent \textbf{Ablation Study.}
%\subsubsection{Ablation Study}
We consider three scenarios to verify the effectiveness of the parsimonious network design in ZSL performance, including 1) only the sample-wise composite linear networks (SCLN) are used; 2) only the maximal margin geometry (MMG) is used (with all base linear networks activated); and 3) full \textit{ParsNets} with both SCLN and MMG. The results are demonstrated in Table~\ref{ablation_study}. 
We can observe that by using only SCLN or MMG separately, the performance is mediocre across all datasets. However, if both SCLN and MMG are used to form the \textit{ParsNets}, our recognition performance is significantly improved with a large margin of 15.4\% in AWA2, 21.5\% in CUB-200, 18.8\% in SUN, and 10.4\% in aPY. 
Such an improvement fully demonstrates the effectiveness and rationality of our method.

\noindent \textbf{Sparseness.}
%\subsubsection{Sparseness}
The sample-wise indicators are another important criterion in our method, where we need to select top-$k$ indicators to activate $k$ (out of total $K$) base linear networks to form the sparse composite linear networks for each sample. 
We rank all variances in Eq.~\ref{final-indicators} and then enable $k$ ranges in \{10, 20, 30, 40, 80, 120, 160, 200\} during training, which corresponds to a sparsity of \{5\%, 10\%, 15\%, 20\%, 40\%, 60\%, 80\%, 100\%\} out of the total number of base linear networks. We repeat the running on all datasets and record the ZSL recognition accuracy of each sparsity in Figure~\ref{sparse}. 
It can be observed that as the number of base linear networks increases, the accuracy gradually improves at the very beginning phase. Soon, it reaches a stable phase where we can explore a trade-off between recognition accuracy and sparsity for each dataset. 
Specifically, we set $k=30$ (15\%) for AWA2, $k=40$ (20\%) for CUB-200, $k=80$ (40\%) for SUN, and $k=30$ (15\%) for aPY.  

% \begin{figure}[h]
%   \centering
% \centerline{\includegraphics[width=0.49\textwidth]{sparse.pdf}}
%   \caption{Sparseness analysis on AWA2, CUB-200, SUN, and aPY datasets. The optimal sparseness is 15\%, 20\%, 40\%, and 15\%, respectively (better viewed in color).}
%   \label{sparse}
% %\vspace{-5mm}
% \end{figure}

\section{Conclusion}
We proposed the \textit{ParsNets}, which is a novel parsimonious-yet-efficient ZSL/GZSL framework. 
%that can be deployed flexibly on resource-constrained devices and achieves equivalent or even better performance against deep models. 
%
Our method first refactors the visual-semantics mapping function into several local base linear networks to estimate the nonlinearity of complex deep models. 
With a constructed maximal margin geometry on the features by enforcing low-rank and high-rank constraints to intra-class and inter-class samples, respectively, our method encourages orthogonal subspaces for different classes which can enable a smooth knowledge transfer between seen/unseen classes.  
The sample-wise indicators are then employed to guarantee a sparse composition of base linear networks to further relieve the domain biased over/under-fittings problem, and meanwhile, enable a lightweight and on-device-friendly ZSL framework. 
Experimental results verified the effectiveness of our method.

\bibliography{my}

\begin{thebibliography}{41}
\providecommand{\natexlab}[1]{#1}
\providecommand{\url}[1]{\texttt{#1}}
\expandafter\ifx\csname urlstyle\endcsname\relax
  \providecommand{\doi}[1]{doi: #1}\else
  \providecommand{\doi}{doi: \begingroup \urlstyle{rm}\Url}\fi

\bibitem[Akata et~al.(2015)Akata, Reed, Walter, Lee, and Schiele]{akata2015evaluation}
Akata, Z., Reed, S., Walter, D., Lee, H., and Schiele, B.
\newblock Evaluation of output embeddings for fine-grained image classification.
\newblock In \emph{IEEE/CVF Conference on Computer Vision and Pattern Recognition}, pp.\  2927--2936, 2015.

\bibitem[Chen et~al.(2023)Chen, Shen, Zhang, and Torr]{chen2023deconstructed}
Chen, D., Shen, Y., Zhang, H., and Torr, P.~H.
\newblock Deconstructed generation-based zero-shot model.
\newblock In \emph{Proceedings of the AAAI Conference on Artificial Intelligence}, volume~37, pp.\  295--303, 2023.

\bibitem[Chen et~al.(2021{\natexlab{a}})Chen, Xie, Liu, Peng, Sun, Li, You, and Shao]{chen2021hsva}
Chen, S., Xie, G., Liu, Y., Peng, Q., Sun, B., Li, H., You, X., and Shao, L.
\newblock Hsva: Hierarchical semantic-visual adaptation for zero-shot learning.
\newblock In \emph{Advances in Neural Information Processing Systems}, pp.\  16622--16634, 2021{\natexlab{a}}.

\bibitem[Chen et~al.(2021{\natexlab{b}})Chen, Luo, Qiu, Wang, Huang, Li, and Zhang]{chen2021semantics}
Chen, Z., Luo, Y., Qiu, R., Wang, S., Huang, Z., Li, J., and Zhang, Z.
\newblock Semantics disentangling for generalized zero-shot learning.
\newblock In \emph{IEEE/CVF International Conference on Computer Vision}, pp.\  8712--8720, 2021{\natexlab{b}}.

\bibitem[Farhadi et~al.(2009)Farhadi, Endres, Hoiem, and Forsyth]{farhadi2009describing}
Farhadi, A., Endres, I., Hoiem, D., and Forsyth, D.
\newblock Describing objects by their attributes.
\newblock In \emph{Computer Vision and Pattern Recognition, 2009. CVPR 2009. IEEE Conference on}, pp.\  1778--1785. IEEE, 2009.

\bibitem[Feng et~al.(2022)Feng, Huang, Yang, Yu, and Sang]{feng2022non}
Feng, Y., Huang, X., Yang, P., Yu, J., and Sang, J.
\newblock Non-generative generalized zero-shot learning via task-correlated disentanglement and controllable samples synthesis.
\newblock In \emph{IEEE/CVF Conference on Computer Vision and Pattern Recognition}, pp.\  9346--9355, 2022.

\bibitem[Fu et~al.(2015)Fu, Hospedales, Xiang, and Gong]{fu2015transductive}
Fu, Y., Hospedales, T.~M., Xiang, T., and Gong, S.
\newblock Transductive multi-view zero-shot learning.
\newblock \emph{IEEE Transactions on Pattern Analysis and Machine Intelligence}, 37\penalty0 (11):\penalty0 2332--2345, 2015.

\bibitem[Guo \& Guo(2020)Guo and Guo]{guo2020novel}
Guo, J. and Guo, S.
\newblock A novel perspective to zero-shot learning: Towards an alignment of manifold structures via semantic feature expansion.
\newblock \emph{IEEE Transactions on Multimedia}, 23:\penalty0 524--537, 2020.

\bibitem[Guo et~al.(2023)Guo, Guo, Zhou, Liu, Lu, and Huo]{guo2023graph}
Guo, J., Guo, S., Zhou, Q., Liu, Z., Lu, X., and Huo, F.
\newblock Graph knows unknowns: Reformulate zero-shot learning as sample-level graph recognition.
\newblock 37\penalty0 (6):\penalty0 7775--7783, 2023.

\bibitem[Huang et~al.(2019)Huang, Wang, Yu, and Wang]{huang2019generative}
Huang, H., Wang, C., Yu, P.~S., and Wang, C.-D.
\newblock Generative dual adversarial network for generalized zero-shot learning.
\newblock In \emph{IEEE/CVF Conference on Computer Vision and Pattern Recognition}, pp.\  801--810, 2019.

\bibitem[Huynh \& Elhamifar(2020)Huynh and Elhamifar]{huynh2020fine}
Huynh, D. and Elhamifar, E.
\newblock Fine-grained generalized zero-shot learning via dense attribute-based attention.
\newblock In \emph{IEEE/CVF Conference on Computer Vision and Pattern Recognition}, pp.\  4483--4493, 2020.

\bibitem[Keshari et~al.(2020)Keshari, Singh, and Vatsa]{keshari2020generalized}
Keshari, R., Singh, R., and Vatsa, M.
\newblock Generalized zero-shot learning via over-complete distribution.
\newblock In \emph{IEEE/CVF Conference on Computer Vision and Pattern Recognition}, pp.\  13300--13308, 2020.

\bibitem[Khan et~al.(2023)Khan, Naeem, Van~Gool, Pagani, Stricker, and Afzal]{khan2023learning}
Khan, M. G. Z.~A., Naeem, M.~F., Van~Gool, L., Pagani, A., Stricker, D., and Afzal, M.~Z.
\newblock Learning attention propagation for compositional zero-shot learning.
\newblock In \emph{IEEE/CVF Winter Conference on Applications of Computer Vision}, pp.\  3828--3837, 2023.

\bibitem[Kumar~Verma et~al.(2018)Kumar~Verma, Arora, Mishra, and Rai]{kumar2018generalized}
Kumar~Verma, V., Arora, G., Mishra, A., and Rai, P.
\newblock Generalized zero-shot learning via synthesized examples.
\newblock In \emph{IEEE/CVF Conference on Computer Vision and Pattern Recognition}, pp.\  4281--4289, 2018.

\bibitem[Lampert et~al.(2013)Lampert, Nickisch, and Harmeling]{lampert2013attribute}
Lampert, C.~H., Nickisch, H., and Harmeling, S.
\newblock Attribute-based classification for zero-shot visual object categorization.
\newblock \emph{IEEE transactions on pattern analysis and machine intelligence}, 36\penalty0 (3):\penalty0 453--465, 2013.

\bibitem[Li et~al.(2020)Li, Zhang, Gong, Yao, and Wu]{li2020field}
Li, Z., Zhang, J., Gong, Y., Yao, Y., and Wu, Q.
\newblock Field-wise learning for multi-field categorical data.
\newblock \emph{Advances in Neural Information Processing Systems}, 33:\penalty0 9890--9899, 2020.

\bibitem[Liu et~al.(2018)Liu, Zhang, Zhao, Sun, and Hoi]{liu2018unified}
Liu, C., Zhang, T., Zhao, P., Sun, J., and Hoi, S.
\newblock Unified locally linear classifiers with diversity-promoting anchor points.
\newblock In \emph{Proceedings of the AAAI Conference on Artificial Intelligence}, volume~32, 2018.

\bibitem[Liu et~al.(2020)Liu, Zhou, Long, Jiang, and Zhang]{liu2020attribute}
Liu, L., Zhou, T., Long, G., Jiang, J., and Zhang, C.
\newblock Attribute propagation network for graph zero-shot learning.
\newblock In \emph{Thirty-Fourth AAAI Conference on Artificial Intelligence}, 2020.

\bibitem[Liu et~al.(2023)Liu, Li, Zhang, Wei, Bai, and Zhao]{liu2023progressive}
Liu, M., Li, F., Zhang, C., Wei, Y., Bai, H., and Zhao, Y.
\newblock Progressive semantic-visual mutual adaption for generalized zero-shot learning.
\newblock In \emph{Proceedings of the IEEE/CVF Conference on Computer Vision and Pattern Recognition}, pp.\  15337--15346, 2023.

\bibitem[Oiwa \& Fujimaki(2014)Oiwa and Fujimaki]{oiwa2014partition}
Oiwa, H. and Fujimaki, R.
\newblock Partition-wise linear models.
\newblock \emph{Advances in Neural Information Processing Systems}, 27, 2014.

\bibitem[Patterson et~al.(2014)Patterson, Xu, Su, and Hays]{patterson2014sun}
Patterson, G., Xu, C., Su, H., and Hays, J.
\newblock The sun attribute database: Beyond categories for deeper scene understanding.
\newblock \emph{International Journal of Computer Vision}, 108\penalty0 (1-2):\penalty0 59--81, 2014.

\bibitem[Pennington et~al.(2014)Pennington, Socher, and Manning]{pennington2014glove}
Pennington, J., Socher, R., and Manning, C.~D.
\newblock Glove: Global vectors for word representation.
\newblock In \emph{Proceedings of the 2014 conference on empirical methods in natural language processing (EMNLP)}, pp.\  1532--1543, 2014.

\bibitem[Qiu \& Sapiro(2015)Qiu and Sapiro]{qiu2015learning}
Qiu, Q. and Sapiro, G.
\newblock Learning transformations for clustering and classification.
\newblock \emph{J. Mach. Learn. Res.}, 16\penalty0 (1):\penalty0 187--225, 2015.

\bibitem[Ranzato et~al.(2007)Ranzato, Boureau, Cun, et~al.]{ranzato2007sparse}
Ranzato, M., Boureau, Y.-L., Cun, Y., et~al.
\newblock Sparse feature learning for deep belief networks.
\newblock \emph{Advances in neural information processing systems}, 20, 2007.

\bibitem[Schonfeld et~al.(2019)Schonfeld, Ebrahimi, Sinha, Darrell, and Akata]{schonfeld2019generalized}
Schonfeld, E., Ebrahimi, S., Sinha, S., Darrell, T., and Akata, Z.
\newblock Generalized zero-and few-shot learning via aligned variational autoencoders.
\newblock In \emph{IEEE/CVF Conference on Computer Vision and Pattern Recognition}, pp.\  8247--8255, 2019.

\bibitem[Sung et~al.(2018)Sung, Yang, Zhang, Xiang, Torr, and Hospedales]{sung2018learning}
Sung, F., Yang, Y., Zhang, L., Xiang, T., Torr, P.~H., and Hospedales, T.~M.
\newblock Learning to compare: Relation network for few-shot learning.
\newblock In \emph{IEEE/CVF Conference on Computer Vision and Pattern Recognition}, pp.\  1199--1208, 2018.

\bibitem[Van~der Maaten \& Hinton(2008)Van~der Maaten and Hinton]{van2008visualizing}
Van~der Maaten, L. and Hinton, G.
\newblock Visualizing data using t-sne.
\newblock \emph{Journal of machine learning research}, 9\penalty0 (11):\penalty0 2579--2605, 2008.

\bibitem[Vyas et~al.(2020)Vyas, Venkateswara, and Panchanathan]{vyas2020leveraging}
Vyas, M.~R., Venkateswara, H., and Panchanathan, S.
\newblock Leveraging seen and unseen semantic relationships for generative zero-shot learning.
\newblock In \emph{European Conference on Computer Vision}, pp.\  70--86. Springer, 2020.

\bibitem[Wah et~al.(2011)Wah, Branson, Welinder, Perona, and Belongie]{WahCUB_200_2011}
Wah, C., Branson, S., Welinder, P., Perona, P., and Belongie, S.
\newblock The caltech-ucsd birds-200-2011 dataset.
\newblock In \emph{Computation \& Neural Systems Technical Report (California Institute of Technology)}, pp.\  1--8, 2011.

\bibitem[Wang et~al.(2018)Wang, Ye, and Gupta]{wang2018zero}
Wang, X., Ye, Y., and Gupta, A.
\newblock Zero-shot recognition via semantic embeddings and knowledge graphs.
\newblock In \emph{IEEE/CVF Conference on Computer Vision and Pattern Recognition}, pp.\  6857--6866, 2018.

\bibitem[Xian et~al.(2017)Xian, Schiele, and Akata]{xian2017zero}
Xian, Y., Schiele, B., and Akata, Z.
\newblock Zero-shot learning-the good, the bad and the ugly.
\newblock In \emph{IEEE/CVF Conference on Computer Vision and Pattern Recognition}, pp.\  4582--4591, 2017.

\bibitem[Xian et~al.(2018{\natexlab{a}})Xian, Lampert, Schiele, and Akata]{xian2018zero}
Xian, Y., Lampert, C.~H., Schiele, B., and Akata, Z.
\newblock Zero-shot learning—a comprehensive evaluation of the good, the bad and the ugly.
\newblock \emph{IEEE transactions on pattern analysis and machine intelligence}, 41\penalty0 (9):\penalty0 2251--2265, 2018{\natexlab{a}}.

\bibitem[Xian et~al.(2018{\natexlab{b}})Xian, Lorenz, Schiele, and Akata]{xian2018feature}
Xian, Y., Lorenz, T., Schiele, B., and Akata, Z.
\newblock Feature generating networks for zero-shot learning.
\newblock In \emph{IEEE/CVF Conference on Computer Vision and Pattern Recognition}, pp.\  5542--5551, 2018{\natexlab{b}}.

\bibitem[Xie et~al.(2019)Xie, Liu, Jin, Zhu, Zhang, Qin, Yao, and Shao]{xie2019attentive}
Xie, G.-S., Liu, L., Jin, X., Zhu, F., Zhang, Z., Qin, J., Yao, Y., and Shao, L.
\newblock Attentive region embedding network for zero-shot learning.
\newblock In \emph{IEEE/CVF Conference on Computer Vision and Pattern Recognition}, pp.\  9384--9393, 2019.

\bibitem[Xie et~al.(2020)Xie, Liu, Zhu, Zhao, Zhang, Yao, Qin, and Shao]{xie2020region}
Xie, G.-S., Liu, L., Zhu, F., Zhao, F., Zhang, Z., Yao, Y., Qin, J., and Shao, L.
\newblock Region graph embedding network for zero-shot learning.
\newblock In \emph{European Conference on Computer Vision}, pp.\  562--580. Springer, 2020.

\bibitem[Xu et~al.(2020)Xu, Xian, Wang, Schiele, and Akata]{xu2020attribute}
Xu, W., Xian, Y., Wang, J., Schiele, B., and Akata, Z.
\newblock Attribute prototype network for zero-shot learning.
\newblock In \emph{Advances in Neural Information Processing Systems}, pp.\  21969--21980, 2020.

\bibitem[Xu et~al.(2022)Xu, Xian, Wang, Schiele, and Akata]{xu2022vgse}
Xu, W., Xian, Y., Wang, J., Schiele, B., and Akata, Z.
\newblock Vgse: Visually-grounded semantic embeddings for zero-shot learning.
\newblock In \emph{IEEE/CVF Conference on Computer Vision and Pattern Recognition}, pp.\  9316--9325, 2022.

\bibitem[Zhang et~al.(2017)Zhang, Xiang, and Gong]{zhang2017learning}
Zhang, L., Xiang, T., and Gong, S.
\newblock Learning a deep embedding model for zero-shot learning.
\newblock In \emph{IEEE/CVF Conference on Computer Vision and Pattern Recognition}, pp.\  2021--2030, 2017.

\bibitem[Zhang \& Saligrama(2015)Zhang and Saligrama]{zhang2015zero}
Zhang, Z. and Saligrama, V.
\newblock Zero-shot learning via semantic similarity embedding.
\newblock In \emph{IEEE International Conference on Computer Vision}, pp.\  4166--4174, 2015.

\bibitem[Zhao et~al.(2022)Zhao, Shen, Wang, and Zhang]{zhao2022boosting}
Zhao, X., Shen, Y., Wang, S., and Zhang, H.
\newblock Boosting generative zero-shot learning by synthesizing diverse features with attribute augmentation.
\newblock In \emph{AAAI Conference on Artificial Intelligence}, volume~36, pp.\  3454--3462, 2022.

\bibitem[Zhu et~al.(2019)Zhu, Xie, Tang, Peng, and Elgammal]{zhu2019semantic}
Zhu, Y., Xie, J., Tang, Z., Peng, X., and Elgammal, A.
\newblock Semantic-guided multi-attention localization for zero-shot learning.
\newblock In \emph{Advances in Neural Information Processing Systems}, pp.\  14917--14927, 2019.

\end{thebibliography}
\bibliographystyle{icml2024}
\end{document}